\documentclass[letterpaper, 10 pt, conference]{ieeeconf}  

\IEEEoverridecommandlockouts                              

\usepackage{amsmath,amssymb,amsfonts}  
\usepackage{mathtools}                 
\usepackage{siunitx}                   

\usepackage{graphicx}                  
\usepackage[caption=false,font=footnotesize]{subfig} 
\usepackage{booktabs}                  
\usepackage{tabularx}                  
\usepackage{multirow}                  
\usepackage{array}                     

\usepackage{algorithm}
\usepackage{algpseudocode}             

\usepackage{cite}                      

\usepackage{bbm}                       
\usepackage{bm}                        
\usepackage{microtype}                 
\usepackage{url}

\newtheorem{theorem}{Theorem}
\newtheorem{lemma}{Lemma}

\title{\LARGE \bf
Planar-Sector LOS Guidance for Interception of Agile Targets with Lifting-Wing Quadcopters
}

\author{Linkai Liu$^{1}$, Kun Yang$^{2}$, Han Zou$^{1}$, Chen Min$^{1}$, Shuli Lv$^{1}$, Shuai Wang$^{1}$ and Quan Quan$^{1}$
\thanks{$^{1}$ School of Automation Science and Electrical Engineering, Beihang University, Beijing, China. 
        \{link\_liu, zouhan2002, min\_chen, lvshuli, wsh\_buaa, qq\_buaa\}@buaa.edu.cn}%
\thanks{$^{2}$ Research and Development Center, China Academy of Launch Vehicle Technology, Beijing, China. 
        yangkun\_buaa@buaa.edu.cn}
}

\begin{document}

\maketitle
\nocite{BSTcontrol}
\thispagestyle{empty}
\pagestyle{empty}

\begin{abstract}

This paper proposes a Planar-Sector Line-of-Sight (PS-LOS) guidance law and an accompanying control stack for lifting-wing quadcopters, enabling robust image-based interception of agile targets. The PS-LOS relaxes conventional conical constraints, preserving maneuverability while reducing aerodynamic penalties. For perception, we employ a delay-compensated Extended Kalman Filter (EKF) to provide low-latency, continuous target estimates. The controller is tailored to lifting-wing quadcopter dynamics and includes coordinated-turn compensation. Outdoor flight experiments demonstrate interceptions against unpredictable agile targets up to 138\,m, validating the method's range and robustness.

\end{abstract}


\section{Introduction\label{sec:Introduction}}

With the rapid expansion of the Unmanned Aerial Vehicle (UAV) industry \cite{floreanoScienceTechnologyFuture2015}, misuse and threats to public safety have become increasingly prominent \cite{ritchieMicroUAVCrime2017}. This trend has created an urgent need for cost-effective, reliable anti-drone solutions, particularly for small, highly maneuverable autonomous drones \cite{Fu2019RaD-VIO}. Such platforms are often compact, communication-silent, and capable of aggressive, noncooperative maneuvers, which significantly reduce the effectiveness of conventional ground-based defenses. Consequently, there is growing interest in interceptor UAVs equipped with onboard sensing and autonomous interception capabilities that can localize, track, and engage non-cooperative targets in real-time.

Research in autonomous drone interception is mainly divided into position-based and Image-Based Visual Servo (IBVS) methods. The former, demonstrated in MBZIRC 2020 \cite{stasinchuk2021multi, vrbaAutonomousCaptureAgile2022}, relies on a perception-planning-control pipeline and requires large drones with GPS and electro-optical pods \cite{taoOptimalTerminalvelocitycontrolGuidance2022}. IBVS simplifies the loop by directly mapping image features into control commands, for instance, the work \cite{yanPreciseInterceptionFlight2025} developed a precise interception framework, highlighting the potential of IBVS for GPS-denied scenarios, while the work \cite{yangHighspeedInterceptionMulticopter2025} demonstrated reliable high-speed interceptions against both stationary and circling maneuvering targets, validating the effectiveness of monocular IBVS in dynamic flight experiments.

However, most studies rely on quadrotors \cite{vrbaAutonomousCaptureAgile2022, yangHighspeedInterceptionMulticopter2025,liThreedimensionalBearingonlyTarget2023}, without a performance gap between the target and the interceptor. Line-of-Sight (LOS) constraints are typically modeled as conic regions centered on the optical axis \cite{yangHighspeedInterceptionMulticopter2025}, which preserve visibility but restrict feasible thrust vectors. Prior validations also focus on constant-velocity or relatively regular circular trajectories \cite{liThreedimensionalBearingonlyTarget2023,yangAutonomousInterceptDrone2020, yangHighspeedInterceptionMulticopter2025}, neglecting evasive targets with unpredictable lateral or vertical accelerations. Recent work integrating IBVS with proportional navigation and FOV-holding control improved interception precision \cite{yanPreciseInterceptionFlight2025}, but still relied on quadrotors and keeps the target near the image center.

\begin{figure}[!t]
  \centering
  \includegraphics[width=0.94\columnwidth,keepaspectratio]{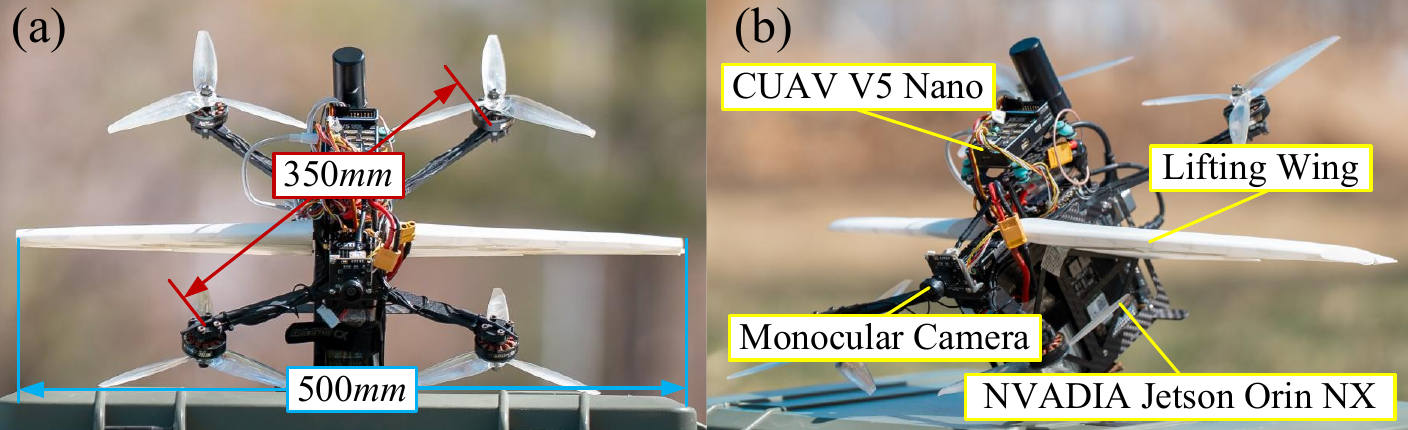}
  \caption{\label{fig:expParameters}Lift-wing quadcopter platform in experiments. (a) Wing span and distance between motors. (b) Core sensing and computing components.}
  \vspace{-0.6cm}
\end{figure}

As a hybrid configuration between conventional quadrotors and tail-sitter UAVs, lifting-wing quadcopter platforms (Fig.~\ref{fig:expParameters}) combine the advantages of both configurations. Compared with quadrotors, they achieve higher flight speeds, lower energy consumption, and substantially longer range \cite{quanLiftingwingQuadcopterModeling2025,zhangPerformanceEvaluationDesign2022}; compared with fixed-wing aircraft, they offer shorter takeoff/landing distances and greater maneuverability; compared with tail-sitters, they provide more agile lateral control and a simpler, more reliable launch and landing procedure \cite{zhangPerformanceEvaluationDesign2022}. Moreover, the work \cite{gaoDodgingLikeABird2023} demonstrated that such platforms can perform aggressive inverted dive maneuvers with stable control and recovery. Consequently, deploying a lifting-wing quadcopter as the interceptor preserves reliability and low deployment complexity while enabling high control agility, faster closing speeds, and a larger interception envelope—often yielding a performance overmatch against quadrotor targets.

Building upon previous IBVS work on quadrotors \cite{yangHighspeedInterceptionMulticopter2025}, we consider a more challenging scenario: replace quadcopters with a quadcopter-surpassing lifting-wing quadcopter \cite{quanLiftingwingQuadcopterModeling2025} with a fixed monocular camera (Fig.~\ref{fig:expParameters}) to achieve longer-range interception of targets exhibiting uncertain lateral and vertical accelerations. This motivates three primary research challenges:
\vspace{-0.01cm}
\begin{enumerate}
\item Restrictive LOS Constraints: FOV-centered conic designs overly restrict the feasible net-acceleration envelope during aggressive maneuvers.
\item Platform Limitation: Quadcopter-based studies cannot exploit the aerodynamic efficiency of lifting-wing quadcopters.
\item Unrealistic Target Maneuvers: Existing validations rarely involve evasive targets with significant lateral or vertical accelerations.
\end{enumerate}

\begin{figure}[htbp]
  \centering
  \includegraphics[width=1.0\columnwidth,keepaspectratio]{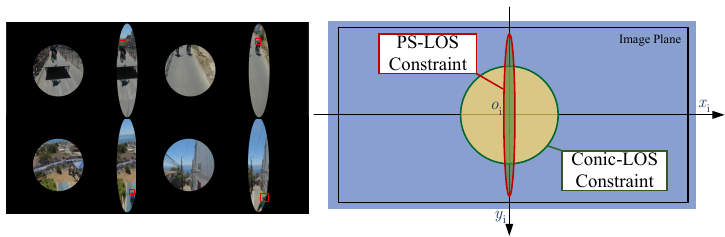}
  \caption{\label{fig:CompareConstraint}The constraint region of two LOS configurations.}
  \vspace{-0.5cm}
\end{figure}

Inspired by the practices of FPV pilots in aggressive visual tracking\footnote{World's Best FPV Drone Shot?, Red Bull, 2022. Available at: \url{https://youtu.be/13OtZFWdhwQ?si=zKxB967XPjZH8pgI}} (Fig.~\ref{fig:CompareConstraint}), we propose constraining the LOS to a planar sector aligned with the camera symmetry plane—tight along the image’s horizontal axis but relaxed along the vertical axis. This strategy preserves target visibility while ensuring sufficient maneuvering authority for lifting-wing quadcopters. Moreover, it directly addresses Challenge~1 (FOV-induced thrust and acceleration limits) and Challenge~3 (irregular lateral and vertical target accelerations), thereby guiding the subsequent control–estimation design.

The main contributions of this work are summarized as follows:
\begin{itemize}
    \item A novel \textbf{Planar-Sector Line-of-Sight (PS-LOS) guidance law} that enlarges the feasible acceleration set, and ensures target visibility during agile maneuvers.
    \item \textbf{System implementation on a lifting-wing quadcopter}, featuring a two-layer controller with coordinated-turn compensation, along with a formal proof of closed-loop stability.
    \item \textbf{Experimental validation on unpredictably agile targets}: outdoor interceptions under irregular lateral and vertical accelerations, achieving successful interceptions up to \textbf{138\,m}, surpassing quadrotor IBVS baselines in range and robustness.
\end{itemize}

Together, these contributions form a unified framework for robustly intercepting  unpredictably agile targets.

\section{Model and Problem Formulation\label{sec:modeling}}

This section formalizes the models and problem statement that constitute the basis of the proposed interception framework, including the coordinate frames, the lifting-wing quadcopter dynamics, the target kinematics, the camera imaging model and the definition of the PS-LOS constraint.

\begin{figure}[htbp]
  \centering
  \includegraphics[width=1.0\columnwidth,keepaspectratio]{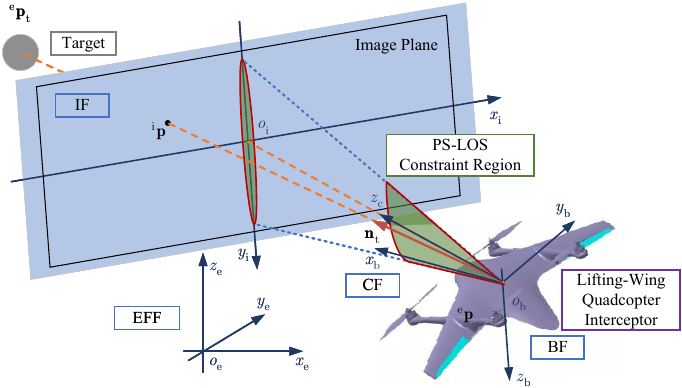}
  \caption{\label{fig:InterceptionCoordinate}Description of the interception problem. The green areas with the red border in IF and CF mark the PS-LOS constraint region designed in this article.}
  \vspace{-0.4cm}
\end{figure}

\subsection{Coordinate Systems and Models}
\subsubsection{Coordinate Systems}

For the vision-based lifting-wing quadcopter interception task, we employ six right-handed coordinate frames, illustrated in Figs.~\ref{fig:InterceptionCoordinate} and \ref{fig:LwCoordinate}:
\begin{itemize}
    \item Earth-fixed frame (EFF) $\{\mathrm{e}\}$: inertial reference for global positions and velocities.
    \item Body frame (BF) $\{\mathrm{b}\}$: quadcopter body-fixed frame describing the aircraft attitude.
    \item Lifting-wing frame (LWF) $\{\mathrm{l}\}$: frame attached to the lifting wing.
    \item Wind frame (WF) $\{\mathrm{w}\}$: local airflow frame used for aerodynamic quantities.
    \item Camera frame (CF) $\{\mathrm{c}\}$: camera pose relative to the vehicle.
    \item Image frame (IF) $\{\mathrm{i}\}$: 2-D image coordinates in the sensor plane.
\end{itemize}
Vectors expressed in frame $\mathcal{F}\in\{\mathrm{e},\mathrm{b},\mathrm{l},\mathrm{w},\mathrm{c},\mathrm{i}\}$ are denoted by the superscript ${}^{\mathcal{F}}\mathbf r$. The rotation matrix $\mathbf{R}_{\mathcal{F}_1}^{\mathcal{F}_2}\in\mathrm{SO}(3)$ denotes the orientation of frame $\mathcal{F}_2$ with respect to frame $\mathcal{F}_1$ (i.e., it rotates vectors from $\mathcal{F}_2$ to $\mathcal{F}_1$).

\begin{figure}[htbp]
  \centering
  \includegraphics[width=1.0\columnwidth,keepaspectratio]{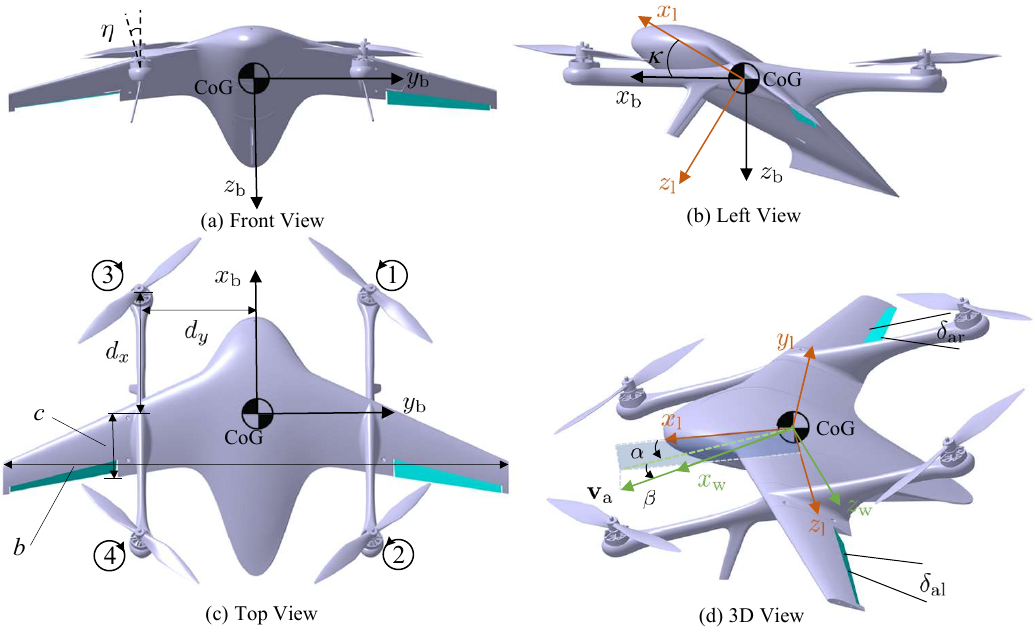}
  \caption{\label{fig:LwCoordinate}Configuration and key properties of the lifting-wing quadcopter\cite{quanLiftingwingQuadcopterModeling2025}. Properties shown: CoG: center of gravity; $\eta$: motor lateral tilt; $b$: wingspan; $c$: mean chord; $\kappa$: wing installation angle; $d_x, d_y$: motor thrust moment arms; $v_a$: airspeed direction; $\alpha, \beta$: angle of attack and sideslip; $\delta_{ar}, \delta_{al}$: right/left aileron deflections.}
  \vspace{-0.2cm}
\end{figure}

The origin of the CF coincides with that of the BF, i.e., the translation vector ${\bf t}_{\rm c}^{\rm b} = {\bf 0}$, and the rotation matrix ${\bf R}_{\rm c}^{\rm b}$ is constant. In this work, the camera is rigidly mounted to the vehicle body. Therefore, interception is considered successful once the camera directly contacts the target.

\subsubsection{Lifting-Wing Quadcopter Model}

The lifting-wing quadcopter model formulated under the special orthogonal group ${\rm SO}(3)$, to describe its dynamics\cite{wangHeadingAdjustmentAdmittance2025}:

\begin{equation}
    \label{eq:LwModel}
    \left\{ \begin{array}{l}
^{\rm{e}}{\dot {\bf{p}}} = {}^{\rm{e}}{\bf{v}},\\
^{\rm{e}}{\dot {\bf{v}}} = ^{\rm{e}}\!\!{\bf{a}} = \frac{1}{m}{^{\rm{e}}\bf{f}},\\
{\dot {\bf{R}}}_{\rm{b}}^{\rm{e}} = {\bf{R}}_{\rm{b}}^{\rm{e}}{\left[ {^{\rm{b}}{\bf{\omega }}} \right]_ \times },\\
\mathbf{J} \, {^{\rm b}\dot{\boldsymbol{\omega}}} = - {^{\rm b}\boldsymbol{\omega}} \times \mathbf{J} \, {^{\rm b}\boldsymbol{\omega}} + {^{\rm b}\mathbf{m}}
\end{array} \right.
\end{equation}
where \( ^{\rm{e}}{\bf{p}} \), \( ^{\rm{e}}{\bf{v}} \) and $^{\rm{e}}{\bf{a}}$ are the position,  velocity and acceleration of the aircraft in the EFF, $\mathbf{J}$ and ${}^{\mathrm{b}}\boldsymbol{\omega}$ are the inertia matrix and angular velocity of the aircraft in the BF, \( m \) is the mass of the aircraft and \( [^{\rm b} \omega]_{\times} \) denotes the skew-symmetric matrix. 
The external forces and moments are given by
\begin{equation}
\left\{
\begin{array}{l}
{}^{\mathrm{e}}\mathbf{f} = \mathbf{R}_{\mathrm{b}}^{\rm e} \cdot {}^{\mathrm{b}}\mathbf{f}_{\mathrm{r}} 
+ \mathbf{R}_{\mathrm{l}}^{\mathrm{e}} \cdot {}^{\mathrm{l}}\mathbf{f}_{\mathrm{a}} 
+ {}^{\mathrm{e}}\mathbf{f}_{g}, \\[4pt]
{}^{\mathrm{b}}\mathbf{m} = {}^{\mathrm{b}}\mathbf{m}_{\mathrm{r}} 
+ \mathbf{R}_{\mathrm{l}}^{\mathrm{b}} \cdot {}^{\mathrm{l}}\mathbf{m}_{\mathrm{a}}.
\end{array}
\right.
\end{equation}
Here \({}^{\mathrm{b}} \mathbf{f}_{\mathrm{r}} = [0\,\,0\,\,f_{\mathrm{r_z}}]^{\rm {T}}\) the force produced by actuators, 
\({}^{\mathrm{e}} \mathbf{f}_{g} = [0\,\,0\,\,mg]^{\rm {T}}\) is the gravitational force, and 
\({}^{\mathrm{l}}\mathbf{f}_{\mathrm{a}}\), \({}^{\mathrm{l}}\mathbf{m}_{\mathrm{a}}\) are the aerodynamic forces and moments generated by the lifting wing, modeled following \cite{wangHeadingAdjustmentAdmittance2025}.

\subsubsection{Target Model}

The target is modeled as a point mass:
\begin{equation}
    \label{equ-TargetModel}
    {}^{\rm e}\dot{\mathbf p}_{\rm t} = {}^{\rm e}\mathbf v_{\rm t}, \qquad
    {}^{\rm e}\dot{\mathbf v}_{\rm t} = {}^{\rm e}\mathbf a_{\rm t}.
\end{equation}
Define the relative states in the EFF as
\begin{equation}
    \label{equ-RelativeParameter}
    {}^{\rm e}\mathbf p_{\rm r} = {}^{\rm e}\mathbf p - {}^{\rm e}\mathbf p_{\rm t}, \quad
    {}^{\rm e}\mathbf v_{\rm r} = {}^{\rm e}\mathbf v - {}^{\rm e}\mathbf v_{\rm t}, \quad
    {}^{\rm e}\mathbf a_{\rm r} = {}^{\rm e}\mathbf a - {}^{\rm e}\mathbf a_{\rm t}.
\end{equation}
The target acceleration ${}^{\rm e}\mathbf a_{\rm t}$ is not directly measured by vision and is treated as a (bounded) disturbance in both controller and estimator design.

\subsubsection{Camera Imaging Model}

Perspective projection maps ${\mathbb R}^3$ to the image plane ${\mathbb R}^2$. Denote the image coordinate by ${}^{\rm i}\mathbf p = [\,{}^{\rm i}p_{x}\,\,{}^{\rm i}p_{y}\,]^{\rm T}$ and the relative position by $\mathbf p_{\rm r}$. The LOS unit vector in the EFF is
\begin{equation}
\label{equ-TargetVector-slim}
\mathbf n_{\rm t} \;=\; -\frac{\mathbf p_{\rm r}}{\|\mathbf p_{\rm r}\|}
\;=\; \mathbf R_{\rm b}^{\rm e}\mathbf R_{\rm c}^{\rm b}
\frac{\big[\,{}^{\rm i}p_{x}\;{}^{\rm i}p_{y}\;f_{\rm oc}\,\big]^{\rm T}}
{\big\|\big[\,{}^{\rm i}p_{x}\;{}^{\rm i}p_{y}\;f_{\rm oc}\,\big]^{\rm T}\big\|},
\end{equation}
where $f_{\rm oc}$ is the camera focal length and $\mathbf R_{\rm c}^{\rm b},\mathbf R_{\rm b}^{\rm e}$ are mounting/attitude rotations.

The sensor field of view (FOV) is defined by horizontal and vertical half-angles $\tfrac{1}{2}\alpha_{\mathrm{hfov}}$ and $\tfrac{1}{2}\alpha_{\mathrm{vfov}}$:
\begin{equation}
\label{equ-FOV}
\resizebox{\columnwidth}{!}{$
\displaystyle
\mathcal C = \big\{{}^{\rm i}\mathbf p \mid
\big|\arctan({}^{\rm i}p_{x}/f_{\rm oc})\big| \le \alpha_{\mathrm{hfov}}/2,\;
\big|\arctan({}^{\rm i}p_{y}/f_{\rm oc})\big| \le \alpha_{\mathrm{vfov}}/2
\big\}
$}
\end{equation}
and the target is observable if ${}^{\rm i}\mathbf p_{\rm t} \in\mathcal C$.

\subsubsection{IBVS Model}

The image Jacobian relates the normalized image velocity to the camera twist:
\begin{equation}
\label{equ-Ls}
{}^{\rm i}\dot{\bar{\bf p}} = {\bf L}_s({}^{\rm i}\bar{\bf p},{}^{\rm c}p_z)\;{}^{\rm c}\tilde{\bf v}, \qquad
{}^{\rm c}\tilde{\bf v}=\begin{bmatrix}{}^{\rm c}{\bf v}\\[2pt]{}^{\rm c}{\boldsymbol\omega}\end{bmatrix},
\end{equation}
with ${}^{\rm i}\bar{\bf p}=[{}^{\rm i}\bar p_x\,\,{}^{\rm i}\bar p_y]^{\rm T}=[{}^{\rm i}p_x/f_{\rm oc}\,\,{}^{\rm i}p_y/f_{\rm oc}]$. The Jacobian ${\bf L}_s \in \mathbb{R}^{2\times6}$ is formulated as \cite{yangHighspeedInterceptionMulticopter2025}.

\subsection{State-Space Representation}

The vehicle attitude is represented by a unit quaternion $\mathbf q\in\mathbb R^4$. Define the state as
\[
\mathbf x = \big[\mathbf q^{\rm T}\;\mathbf p_{\rm r}^{\rm T}\;\mathbf v_{\rm r}^{\rm T}\;{}^{\rm i}\bar{\mathbf p}^{\rm T}\;\mathbf b_{\rm gyr}^{\rm T}\;\mathbf b_{\rm acc}^{\rm T}\big]^{\rm T}\in\mathbb R^{18},
\]
comprising attitude, relative position and velocity (Eq.~\eqref{equ-stateTrans-slim}), normalized image feature, and IMU biases.

The discrete-time state update used in the estimator is compactly written as
\begin{equation}
\label{equ-stateTrans-slim}
\hat{\mathbf x}_{k+1} = \mathbf f(\mathbf x_k,\mathbf u_k,\mathbf w_k),
\end{equation}
where $\mathbf u_k=[\boldsymbol\omega_{{\rm gyr},k}^{\rm T}\,\;\mathbf a_{{\rm acc},k}^{\rm T}]^{\rm T}$ (measured IMU inputs) and process noise
\begin{equation}
        \label{equ-ProcessNoise}
     \mathbf w_k \sim \mathcal{N}(\mathbf 0,\mathbf Q_k),\qquad
\mathbf Q_k=\operatorname{diag}(\mathbf Q_\omega,\mathbf Q_a).   
\end{equation}
\begin{equation}
\mathbf w_k = \begin{bmatrix}\mathbf n_{b_{\rm gyr}}^{\rm T} & \mathbf n_{b_{\rm acc}}^{\rm T}\end{bmatrix}^{\rm T}.
\end{equation}

IMU biases are modeled as random-walk (Wiener) processes, and measurements are corrected by subtracting the estimated bias (compact form):
\[
\begin{aligned}
{}^{\rm e}\boldsymbol\omega &= \mathbf R_{\rm b}^{\rm e}(\boldsymbol\omega_{\rm gyr}-\mathbf b_{\rm gyr}), \qquad
{}^{\rm e}\mathbf a &= \mathbf R_{\rm b}^{\rm e}(\mathbf a_{\rm acc}-\mathbf b_{\rm acc}),
\end{aligned}
\]
with $\dot{\mathbf b}_{\rm gyr}=\mathbf n_{\rm gyr},\ \dot{\mathbf b}_{\rm acc}=\mathbf n_{\rm acc}$, where $\mathbf{n}_{\rm gyr}$ and $\mathbf{n}_{\rm acc}$ are zero-mean Gaussian random vectors.

\subsection{PS-LOS: Theoretical Results and Motivation}
\label{sec-ControlStrategy}

\subsubsection{Motivation}
In the previous work \cite{yangHighspeedInterceptionMulticopter2025}, the LOS constraint is enforced as a conical region about the camera optical axis with an apex angle $\alpha_{\rm cone}$ of roughly $40^\circ$:
\begin{equation}
\label{equ-cone-set}
\mathcal{S}_{\mathrm{cone}} = 
\Bigl\{ \mathbf{n} \in \mathbb S^2 \;\Big|\; 
\cos(\alpha_{\mathrm{cone}}) < \bigl| \mathbf{n}_{\rm d}^{\mathrm{T}} \mathbf{n} \bigr| \le 1 
\Bigr\},
\end{equation}
where $\mathbb S^2$ denotes the two-dimensional manifold of unit vectors in $\mathbb R^3$ and $\mathbf{n}_{\rm d}$ is the desired LOS direction.
While this cone keeps the target near the image center and preserves some maneuverability, it substantially restricts the actuator force envelope---reducing the available maximum thrust by up to approximately 50\% for a large subset of LOS directions, i.e., the directions most required during aggressive target maneuvers (see Fig.~\ref{fig:CompareThrustBall}(a)). 

Empirical observations of elite FPV pilots show a consistent strategy: when pursuing agile targets—accelerating, decelerating, climbing, descending, or turning—they keep the target locked near the FOV centerline, enforcingtight alignment along the image’s horizontal axis while allowing substantial deviation along the vertical axis (Fig.~\ref{fig:CompareConstraint}). However, the conventional conic constraint used in interception control defined in \cite{yangHighspeedInterceptionMulticopter2025}, which doesn't capture this asymmetric tolerance, thus limiting maneuvering freedom near the FOV edges.

\noindent\textbf{Definition 1 (PS-LOS).}
The PS-LOS is defined in CF as
\begin{equation}
\label{equ-PSLOS}
\resizebox{\columnwidth}{!}{$
\displaystyle
\mathcal{S}_{\mathrm{PS}}=\Big\{\mathbf n\in\mathbb S^2 \;\Big|\;
\big|\mathbf n_{\rm hd}^{\mathrm T}\mathbf n\big|<\sin (\alpha_{\mathrm{lon}}),\;\;
\big|\mathbf n_{\rm vd}^{\mathrm T}\mathbf n\big|\le\sin(\alpha_{\mathrm{lat}})
\Big\},
$}
\end{equation}
where $\mathbf n$ is the LOS unit vector, and the two vectors are defined as $\mathbf{n}_{\rm hd} = [0\,\, 1\,\, 0]^{\rm T}\in\mathbb S^2$ and $\mathbf{n}_{\rm vd} = [1\,\, 0\,\, 0]^{\rm T}\in\mathbb S^2$ in CF. The horizontal constraint $\alpha_{\mathrm{lon}}=\alpha_{\mathrm{hfov}}/2-\alpha_{\mathrm{safe}}>\alpha_{\mathrm{cone}}$ (relaxed toward the horizontal FOV edges), while The vertical constraint $\alpha_{\mathrm{lat}}\le\varepsilon$ ($\varepsilon\!\to\!0$). In IF the PS-LOS constraint region shows in Fig.~\ref{fig:InterceptionCoordinate}.


\noindent\textbf{Remark 1.} The role of the LOS constraint region is to keep the target features within this region throughout the interception process; it acts as a limitation on the controller \textbf{rather than a restriction on the visual perception range}.

\begin{figure}[htbp]
  \centering
  \includegraphics[width=1.0\columnwidth,keepaspectratio]{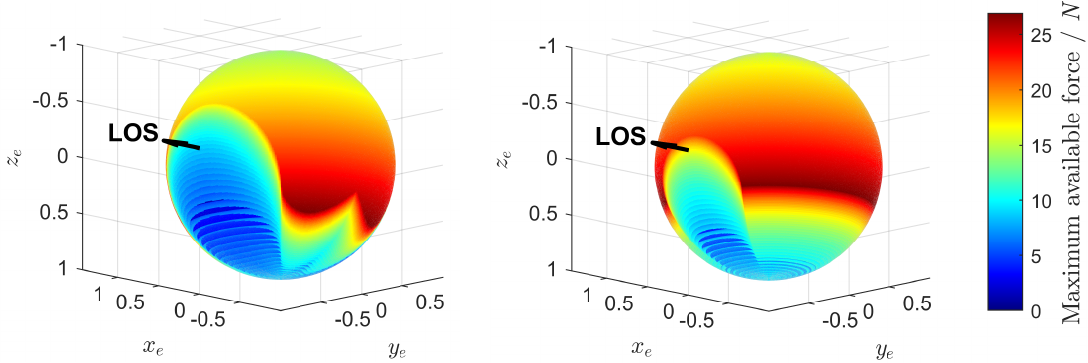}
  \caption{\label{fig:CompareThrustBall}Comparison of maximum available net force ${}^{\mathrm{e}}\mathbf{f}$ for a lifting-wing quadcopter (maximum thrust 25\,N, mass 1\,kg) under two LOS constraints. Each point on the sphere is color-coded by the magnitude of the maximum achievable ${}^{\mathrm{e}}\mathbf{f}$ in that direction. (a) LOS constrained to a cone about the camera optical axis of $40^\circ$. (b) LOS constrained to a planar sector confined to the camera symmetry plane (horizontal constraint relaxed toward the FOV edges of $55^\circ$, lateral constraint tightened).}
  \vspace{-0.5cm}
\end{figure}

\subsubsection{Enhanced Agility through PS-LOS}

A key insight of this work is that only two attitude angles \textemdash{roll $\phi$ and pitch $\theta$} are sufficient to steer the acceleration toward any desired direction, while maintaining LOS of the target in the Planer-Sector region (which imposes instantaneous yaw limits). This observation forms the theoretical foundation of our PS-LOS design.

\begin{lemma}[cf.~\cite{junkins1993geometry}]
Let $\mathbf{e}_{\rm t}=[0\,\,0\,\,-1]^{\rm T}$ and let $R_x(\phi)$, $ R_y(\theta)\in {\rm SO}(3)$ denote rotations about the body $x$- and $y$-axes respectively. 
For any unit vector $\mathbf{t}=[t_x\,\,t_y\,\,t_z]^{\rm T}\in\mathbb{R}^3$ there exist angles $\theta,\phi$ (except the degenerate case $t_x=\pm1$) such that
\[
\mathbf{t}=R_x(\phi)R_y(\theta)\,\mathbf{e}_z.
\]
One convenient choice is
\[
\theta=-\arcsin(t_x),\qquad \phi=\operatorname{atan2}(t_y,-t_z),
\]
with the usual convention for \(\operatorname{atan2}\). 
\end{lemma}

\noindent\textbf{Remark 2.}  
\textit{Lemma 1} shows that two orthogonal attitude angles (roll $\phi$ and pitch $\theta$) suffice to align the thrust vector $\mathbf{e}_{\rm t}$ with any target direction satisfying $t_x \neq \pm 1$. This geometric result provides a rigorous foundation for the subsequent PS-LOS guidance law.

We note that, due to the limitation of FOV (\ref{equ-FOV}), the LOS is constrained within a sector rather than completely circular along the body symmetry plane with a fixed angular span. Within this sector, although the allowable $\theta$ range limits the direct orientation of the thrust vector ${\mathbf{f}_{\mathrm{r}}}$, the combined effect of $\mathbf{f}_{g}$, and $\mathbf{m}_{\mathrm{a}}$ allows the adequate net force ${}^{\mathrm{e}}\mathbf{f}$ acting on the vehicle to cover a sufficiently large range of translational accelerations (see Fig.~\ref{fig:CompareThrustBall}(b)) for typical interception tasks (see Fig.~\ref{fig:ControlStrategy}) (in directions that thrust cannot directly achieve, the available maximum ${}^{\mathrm{e}}\mathbf{f}$ is reduced to different extents but remains effective through PS-LOS).

\begin{figure}[htbp]
  \centering
  \includegraphics[width=1.0\columnwidth,keepaspectratio]{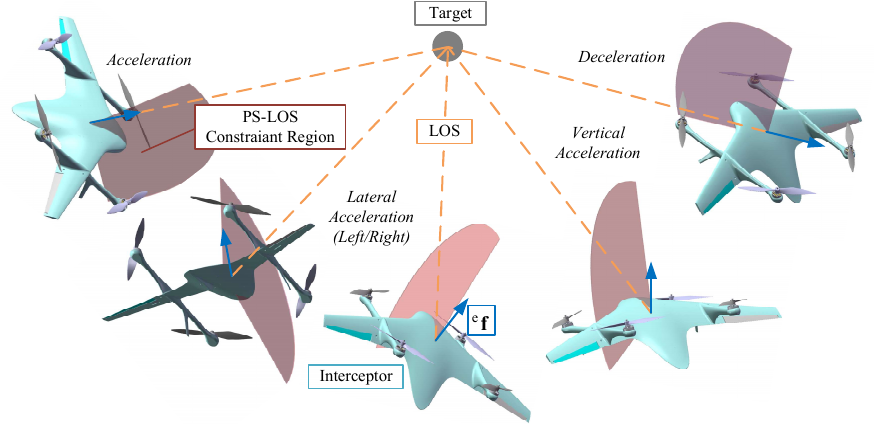}
  \caption{\label{fig:ControlStrategy}%
  Maneuverability of lifting-wing quadcopters under PS-LOS constraints.}
  \vspace{-0.2cm}
\end{figure}

\subsubsection{Theoretical Benefit}

The proposed PS-LOS brings two main theoretical benefits:

\textbf{Stronger bearing regulation and agility, robust to disturbances.} 
By tightening the vertical bound while relaxing the horizontal bound of the LOS region (see Fig.~\ref{fig:CompareConstraint}), the PS-LOS enforces precise bearing control and enhances agility along the LOS direction; meanwhile, stronger corrections along the tangent of the LOS manifold improve robustness against external perturbations (e.g., wind gusts), enabling persistent, low-latency tracking under aggressive target maneuvers.

\textbf{Aerodynamic efficiency on lifting-wing platforms.} 
For lifting-wing quadcopter dynamics, aerodynamic loads and actuator forces concentrate near the vehicle symmetry plane (sideslip maneuvers are inefficient and energy-wasting) \cite{quanLiftingwingQuadcopterModeling2025}. The PS-LOS confines the required maneuvering forces within the symmetry plane (see Fig.~\ref{fig:ControlStrategy}), reducing aerodynamic penalties while preserving maneuverability.

\subsubsection{Model for Control}
Modeling the PS-LOS constraint yields the following control errors:
\begin{equation}
    \label{equ-z_1,z_2}
    \left\{ 
    \begin{array}{l}
        z_1 = \big|\mathbf{n}_{\rm hd}^{\rm T}\mathbf{n}_{\rm t}\big| \le c_h\\[2pt]
        z_2 = \mathbf{n}_{\rm vd}^{\rm T}\mathbf{n}_{\rm t}
    \end{array} 
    \right.,
\end{equation}
where $\mathbf{n}_{\rm t}$ is the LOS unit vector defined in \eqref{equ-TargetVector-slim}. The constant ${c_h} = \sin (\alpha_{\mathrm{lon}})$ ($\alpha_{\mathrm{lon}}$ defined as Eq.~\eqref{equ-PSLOS}).

\begin{figure*}[htbp]
  \centering
  \includegraphics[width=0.95\textwidth]{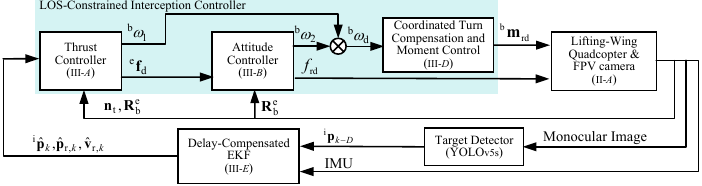}
  \caption{\label{fig:ControllerStructure}Interception controller structure.}
  \vspace{-0.6cm}
\end{figure*}

\subsection{Controller and Estimator Objectives}
\label{sec-ControllerEstimatorObjectives}

We consider a lifting-wing quadcopter interceptor pursuing an agile target with irregular lateral and vertical accelerations; let the relative states be ${}^{\rm e}\mathbf p_{\rm r},{}^{\rm e}\mathbf v_{\rm r}$ and the LOS unit vector $\mathbf n_{\rm t}$.
The interception time is $T^\star=\inf\{t\ge 0\mid \|{}^{\rm e}\mathbf p_{\rm r}(t)\|\le r_{\rm hit}\}$ with capture radius $r_{\rm hit}>0$. Assume $\|{}^{\rm e}\mathbf a_{\rm t}(t)\|\le a_{\max}$ and sufficient actuation $f_{\rm rd}(t)>mg+a_{\max}$. 

The objectives are: (i) sector invariance $\big|\mathbf{n}_{\rm hd}^{\rm T}\mathbf{n}_{\rm t}\big| \le c_h$ with $z_2(t)\to 0$; (ii) interception ${}^{\rm e}\mathbf p_{\rm r}(t)\to \mathbf 0$ (i.e., $T^\star<\infty$).

To achieve this, we employ a PS-LOS guidance law (Eq.~\eqref{equ-z_1,z_2}) with a two-layer controller (Fig.~\ref{fig:ControllerStructure}); stability follows from a composite Lyapunov analysis (cf. Eq.~\eqref{equ-L_dot}). Target-state low-latency feedback is provided by a delay-compensated EKF (Fig.~\ref{fig:EKF}).

\section{Interception Control based on Planar-Sector Guidance\label{sec:guidance}}

This section develops the interception controller leveraging the PS-LOS constraint introduced in Section~\ref{sec-ControlStrategy}. We translate the PS-LOS objectives into an outer-loop thrust/LOS law and an inner-loop attitude tracker with coordinated-turn compensation, and provide a Lyapunov-based analysis of closed-loop safety and convergence. Implementation details and pseudocode are summarized in Fig.~\ref{fig:ControllerStructure} and \textit{Algorithm~\ref{alg:ControllerFlow}}.

\subsection{Outer-Loop Thrust Control}
\label{sec-OuterLoop}

Building on the PS-LOS constraint above, we design an outer-loop thrust law that (i) enforces the PS-LOS and (ii) drives the vehicle to the target. Recall the LOS errors $z_1,z_2$ (Eq.~\eqref{equ-z_1,z_2}) and define the relative-position error ${\bf z}_3={}^{\rm e}{\bf p}_{\rm r}$ alongside with the velocity tracking error ${\bf z}_4={\bf v}_{\rm r}-{\bf v}_{\rm rd}$ (${\bf v}_{\rm rd}=-c_1{}^{\rm e}{\bf p}_{\rm r}$, $c_1>0$ is a tuning gain.). We adopt compact Lyapunov terms together with a Barrier Lyapunov Function \cite{teeBarrierLyapunovFunctions2009}:
\[
L_1=\tfrac12\log\!\frac{c_h^2}{c_h^2-z_1^2},\quad
L_2=\tfrac12 z_2^2,\quad
L_3=\tfrac12{\bf z}_3^{\rm T}{\bf z}_3,
\]
and the composite candidate 
\begingroup
\setlength{\abovedisplayskip}{4pt}
\setlength{\belowdisplayskip}{4pt}
\begin{equation}
\label{equ-L_4}
L_4 = L_1 + L_2 + \tfrac{1}{2}\mathbf z_4^{\rm T}\mathbf z_4.
\end{equation}
\endgroup

Introduce
$K_h={z_1}/{(c_h^2-z_1^2)}$ ($c_h$ defined as Eq.~\eqref{equ-z_1,z_2}), $K_v=z_2$,
which appear naturally in the time-derivatives of $L_1$ and $L_2$ and serve as barrier gains.

Design an implementable outer-loop command that yields $\dot L_4<0$ is
\begin{equation}
\begin{aligned}
{}^{\rm e}{\bf f}_{\rm rd} =& -\mathbf{R}_{\rm l}^{\rm e}\,{}^{\rm l}{\bf f}_{\rm a} - {}^{\rm e}{\bf f}_g - m\Big(-c_1{}^{\rm e}{\bf v}_{\rm r} - c_2{\bf z}_4 - {}^{\rm e}{\bf p}_{\rm r} \\
& - \frac{K_h}{\|{\bf p}_{\rm r}\|}\big(I-{\bf n}_{\rm t}{\bf n}_{\rm t}^{\rm T}\big){\bf n}_{\rm hd} - \frac{K_v}{\|{\bf p}_{\rm r}\|}\big(I-{\bf n}_{\rm t}{\bf n}_{\rm t}^{\rm T}\big){\bf n}_{\rm vd}\Big),
\end{aligned}
\label{equ-alpha1}
\end{equation}
alongside with the LOS-convergence angular-rate command
\begin{equation}
{}^{\rm b}{\boldsymbol\omega}_1 = c_\omega\,{{\bf R}_{\rm b}^{\rm e}}^{\rm T}\!\Big(K_h({\bf n}_{\rm t}\times{\bf n}_{\rm hd})+K_v({\bf n}_{\rm t}\times{\bf n}_{\rm vd})\Big).
\label{equ-omega_1}
\end{equation}

This desired angular velocity component term is used for LOS convergence, where $c_2,c_\omega>0$ are other two tuning gains. These gains ($c_1,c_2,c_\omega$) can be freely adjusted according to different control requirements, and in our experiments, they were varied within the range of $0.3\sim2.0$. Intuitively, the terms $(I-{\bf n}_{\rm t}{\bf n}_{\rm t}^{\rm T})$ produce corrections tangent to the LOS manifold and realize the barrier effect on $z_1,z_2$, while the remaining terms regulate approach rate and damping. The outer loop, therefore, preserves the LOS sector and produces the desired thrust direction/magnitude for the inner-loop attitude controller to track.

\subsection{Inner-Loop Attitude Control}
\label{sec-InnerLoop}

The inner loop tracks the outer-loop thrust command. Let
\begin{equation}
        \mathbf n_{\rm fd} = \frac{{}^{\rm e}\mathbf f_{\rm rd}}{\|{}^{\rm e}\mathbf f_{\rm rd}\|},
\end{equation}
and choose $\mathbf R_{\rm d}$ to align the body thrust axis $\mathbf n_{\rm f}=\mathbf R_{\rm b}^{\rm e}\mathbf e_3$ ($e_3=[0\,\,0\,\,1]^{\rm T}$) with $\mathbf n_{\rm fd}$. A Rodrigues construction is:
\begingroup
\setlength{\abovedisplayskip}{4pt}
\setlength{\belowdisplayskip}{4pt}
\setlength{\abovedisplayshortskip}{3pt}
\setlength{\belowdisplayshortskip}{3pt}
\[
\mathbf r = \frac{\mathbf n_{\rm f}\times\mathbf n_{\rm fd}}{\|\mathbf n_{\rm f}\times\mathbf n_{\rm fd}\|},\quad
\phi_r = \arccos(\mathbf n_{\rm f}^{\rm T}\mathbf n_{\rm fd}).
\]
Then $\mathbf R_{\rm tilt}= \mathbf I+\sin\phi_r[\mathbf r]_\times+(1-\cos\phi_r)[\mathbf r]_\times^2$ and $\mathbf R_{\rm d}=\mathbf R_{\rm tilt}\,\mathbf R_{\rm b}^{\rm e}$.

Based on the attitude error $\mathbf z_\omega = \tfrac12\operatorname{vex}\big(\mathbf R_{\rm d}^{\rm T}\mathbf R_{\rm b}^{\rm e}-{\mathbf R_{\rm b}^{\rm e}}^{\rm T}\mathbf R_{\rm d}\big)$, define
\begin{equation}
\label{equ-L_5}
L_5 = \Psi(\mathbf R_{\rm d}, \mathbf R_{\rm b}^{\rm e})
    = \tfrac{1}{4}\big\| \mathbf I - \mathbf R_{\rm d}^{\rm T}\mathbf R_{\rm b}^{\rm e} \big\|_{\rm F}^2.
\end{equation}
Then, design inner-loop feedback that yields $\dot L_5\le 0$ as
\begin{equation}
\label{equ-omega2}
{}^{\rm b}\boldsymbol\omega_2 = -\,c_\omega\,\mathbf z_\omega,\qquad c_\omega>0,
\end{equation}
which gives
\begin{equation}
\label{equ-L_5_dot}
\dot L_5 = -\,c_\omega\,\mathbf z_\omega^{\rm T}\mathbf z_\omega \;\le\; 0.
\end{equation}
\endgroup

Combining this with the LOS-convergence term \({}^{\rm b}\boldsymbol\omega_1\) from Eq.~\eqref{equ-omega_1} and applying actuator saturation to obtain the commanded signals yields:
\begin{equation}
    \label{equ-ControlLaw}
    \left\{
    \begin{aligned}
    f_{\rm rd} &= \|{}^{\rm e}\mathbf f_{\rm rd}\| \\
    {}^{\rm b}\boldsymbol\omega_{\rm d} &= {}^{\rm b}\boldsymbol\omega_1+{}^{\rm b}\boldsymbol\omega_2
    \end{aligned}
    \right..
\end{equation}

\subsection{Stability Analysis}
\label{sec-StabilityAnalysis}

This subsection formalizes the closed-loop safety and convergence properties of the proposed PS-LOS guidance and two-layer controller. The theorem below establishes sector invariance and asymptotic convergence of the relative position.
\begin{theorem} [Constraint Satisfaction and Convergence]
  \label{theorem1}
If $\big|\mathbf{n}_{\rm hd}^{\rm T}\mathbf{n}_{\rm t}(0)\big| \le c_h$, then under the control law in \eqref{equ-ControlLaw}, the closed-loop trajectory remains in $\mathcal{S}_{\rm {PS}}$ and $\mathbf{p}_{\rm r}\to \mathbf{0}$ as $t\to\infty$.
\end{theorem}

\noindent\textit{Proof.}
Consider the composite Lyapunov function $L=L_4+L_5$, $L_4$ and $L_5$ defined in  Eq.~\eqref{equ-L_4} and Eq.~\eqref{equ-L_5}.
Substituting Eq.~\eqref{equ-alpha1} into $\dot L_4$ and combining with Eq.~\eqref{equ-L_5_dot}, and using the body-rate command in Eq.~\eqref{equ-ControlLaw}, yield
\begin{equation}
  \label{equ-L_dot}
\dot L \le -\,c_1\,\mathbf{p}_{\rm r}^{\rm T}\mathbf{p}_{\rm r}
          -\,c_2\,\mathbf{z}_4^{\rm T}\mathbf{z}_4
          -\,c_\omega\,\mathbf{z}_\omega^{\rm T}\mathbf{z}_\omega \;\le\; 0,
\end{equation}
with positive gains $c_1,c_2,c_\omega$.
Hence $L$ is nonincreasing and bounded below, implying forward invariance of $\big|\mathbf{n}_{\rm hd}^{\rm T}\mathbf{n}_{\rm t}(t)\big| \le c_h$ and boundedness of the closed-loop signals.
By standard LaSalle/Barbalat arguments \cite{Khalil2002}, $\mathbf{z}_\omega\to \mathbf{0}$ and $\mathbf{z}_4\to \mathbf{0}$, and the translational error dynamics then imply $\mathbf{p}_{\rm r}\to \mathbf{0}$. This proves safety and convergence. The detailed expansion and proof process is similar to that in \cite{yangHighspeedInterceptionMulticopter2025} and is omitted here for brevity. \hfill$\blacksquare$

\subsection{Coordinated-Turn Compensation and Moment Control}
\label{sec-LwController}

With high flight speed, sideslip reduces aerodynamic efficiency for lifting-wing quadcopters. We therefore add a coordinated-turn correction that blends a desired yaw rate for sideslip compensation with the nominal attitude command. The coordinated yaw rate is $\dot\psi_{\rm ct}={(g\tan\phi)}/{V}$, where $V$ is vehicle speed. Then the coordinated-turn angular-rate correction is defined as
\begin{equation}
\Omega_{\rm ct}=\begin{bmatrix}0 & 0 & g\tan\phi/V\end{bmatrix}^{\rm T}.
\end{equation}

The coordinated-turn correction is weighted by a speed-dependent factor \cite{quanLiftingwingQuadcopterModeling2025}:
\begin{equation}
w=\operatorname{sat}\!\Big(\frac{V-V_{\min}}{V_{\max}-V_{\min}},0,1\Big),
\end{equation}
and the desired body angular-rate command is then given by
\begin{equation}
\label{equ-Omega_ct}
{}^{\rm b}\boldsymbol\omega_{\rm ct,d}={}^{\rm b}\boldsymbol\omega_{\rm d} + w\,\Omega_{\rm ct}\,.
\end{equation}

Finally, the moment command that includes coordination compensation is implemented as a saturated PID-style law\cite{quanLiftingwingQuadcopterModeling2025}:
\begin{equation}
\label{equ-m_d}
\begin{aligned}
{}^{\rm b}\mathbf m_{\rm rd} = &-\mathbf R_{\rm l}^{\rm b}\,{}^{\rm l}\mathbf m_{\rm a} + \mathbf J\big(-\mathbf K_{\rm p}({}^{\rm b}\boldsymbol\omega-{}^{\rm b}\boldsymbol\omega_{\rm ct,d})\\
&-\mathbf m_{\rm d,1}-\mathbf K_{\rm d}({}^{\rm b}\dot{\boldsymbol\omega}-{}^{\rm b}\dot{\boldsymbol\omega}_{\rm ct,d})\big),
\end{aligned}
\end{equation}
with ${{\bf{m}}_{\rm{d},\rm{l}}} =  {{\bf{K}}_{\rm{i}}} \int \left( {^{\rm{b}}}{\bf{\omega }} - {{}^{\rm b}{\bf{\omega }}_{\rm{ct}, \rm{d}}} \right) \, \mathrm{d}t$ the integral term and $\mathbf K_{\rm p},\mathbf K_{\rm i},\mathbf K_{\rm d}$ PID gains. Here \(\operatorname{sat}(\cdot)\) performs vector saturation as defined in \cite{quanLiftingwingQuadcopterModeling2025}. This formulation explicitly compensates for coordinated turn requirements and preserves aerodynamic efficiency during high-speed maneuvers.

\begin{algorithm}[htbp]
\caption{Interception Controller}
\label{alg:ControllerFlow}
\begin{algorithmic}[1]
\Require Image ${}^{\rm i}\bar{\mathbf{p}}$, velocity $V$, vehicle attitude
\Ensure Desired $f_{\rm rd}$, ${}^{\rm b}\mathbf{m}_{\rm rd}$
\State Initialize control parameters
\For{$t = 1$ \textbf{to} $T^\star$}
    \State Acquire ${}^{\rm i}\bar{\mathbf{p}}$, estimate $\hat{\mathbf{p}}_{\rm r}, \hat{\mathbf{v}}_{\rm r}$ via \textit{Algorithm~\ref{alg:DelayedFilter}}
    \State Compute $f_{\rm rd}, {\mathbf{\omega}}_{\rm d}$ (Eq.~\eqref{equ-ControlLaw})
    \If{$V \ge V_{\rm min}$} \State Compute ${}^{\rm b}\boldsymbol{\omega}_{\rm ct}$ (Eq.~\eqref{equ-Omega_ct}) \EndIf
    \State Compute ${}^{\rm b}\mathbf{m}_{\rm rd}$ (Eq.~\eqref{equ-m_d})
\EndFor
\end{algorithmic}
\end{algorithm}

\subsection{Delay-Compensated EKF}
\label{sec:estimation}

Onboard sensor and vision delays ($0.13\sim0.16 \, \rm{s}$) can induce large estimation errors for highly maneuvering targets (e.g., a $1.2 \, \rm{m}$ shift at $8 \, \rm{m/s}$), risking target loss during interception. To mitigate this issue, we adopt a delay-compensated Extended Kalman Filter (DC-EKF) that fuses high-rate IMU propagation with delayed monocular feature updates. The overall workflow is illustrated in Fig.~\ref{fig:EKF}, while the recursive filtering logic is summarized in \textit{Algorithm~\ref{alg:DelayedFilter}}.
\begin{figure}[htbp]
  \centering
  \includegraphics[width=0.8\columnwidth,keepaspectratio]{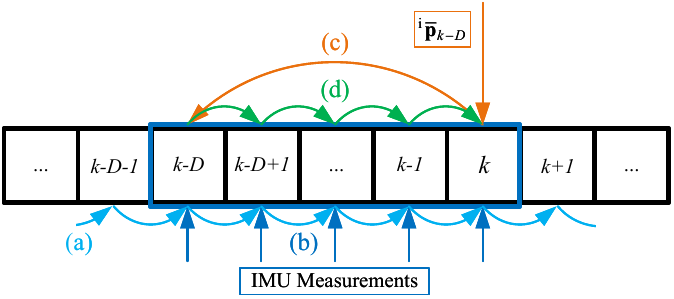}
  \caption{\label{fig:EKF}Workflow of the DC-EKF. (a) Nominal IMU propagation when no image features are available. (b) Maintaining an IMU data buffer for later use in delay compensation. (c) EKF update with delayed visual features at time \(t_{k-D}\). (d) Re-propagation of the corrected state from \(t_{k-D}\) to the current time \(t_k\) using the buffered IMU measurements.}
  \vspace{-0.2cm}
\end{figure}

\begin{figure*}[htbp]
  \centering
  \includegraphics[width=\textwidth]{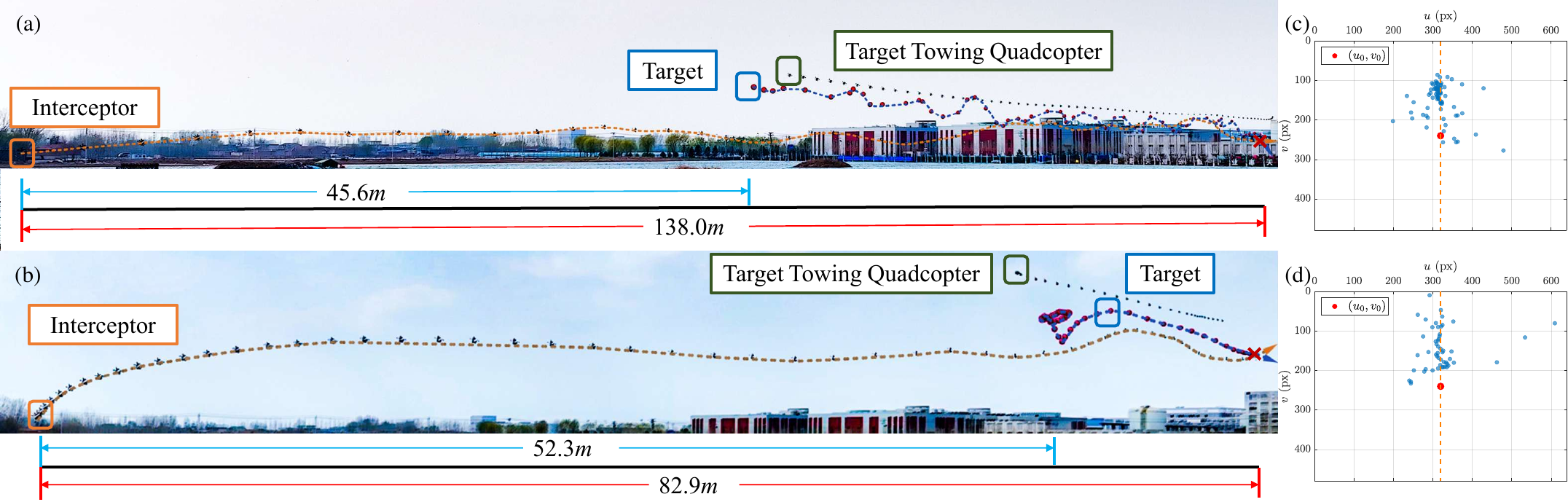}
  \\[2pt]
  \includegraphics[width=\textwidth]{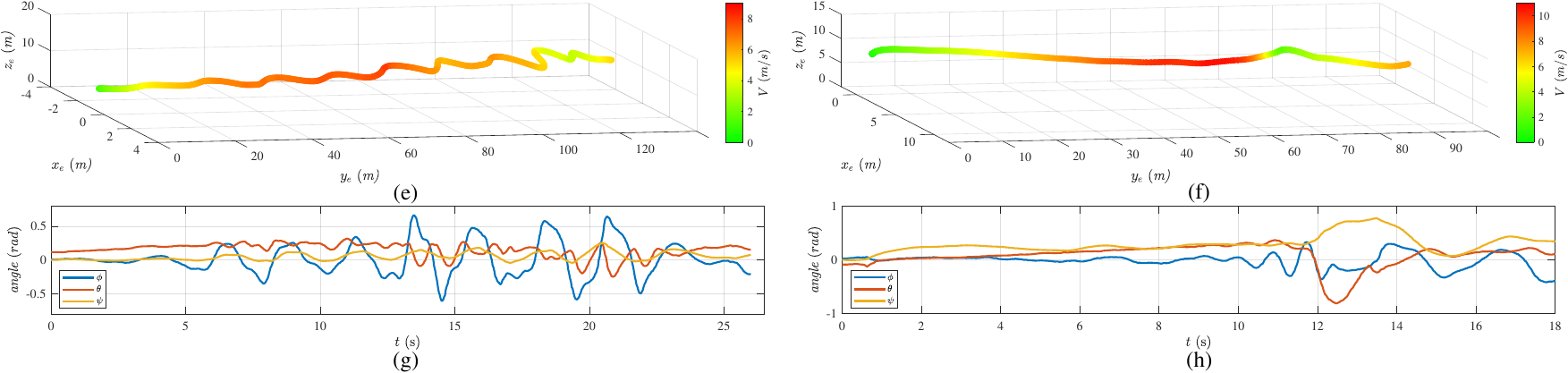}
  \caption{Experimental results: (a) Trajectories for the interception lasting 138\,m. (b) Trajectories for the interception lasting 89\,m. (c,d) image-feature. (e,f) Trajectory and velocity profiles for two trails (max 8.9 / 10.8\,m/s). (g,h) Attitude variation. }
  \label{fig:expFrames}
  \vspace{-0.4cm}
\end{figure*}

\noindent\textbf{Remark 3.} 
The proposed DC-EKF generalizes the DKF in \cite{yangHighspeedInterceptionMulticopter2025}. While both filters compensate for image-processing latency by re-propagating states from delayed updates, our DC-EKF (i) employs nonlinear state propagation (Eq.~\eqref{equ-stateTrans}) instead of a linearized model, (ii) introduces refined quaternion update rules (Eq.~\eqref{equ-delta_q}), and (iii) corrects several omissions in the state-transition equations. These refinements ensure stable estimation for highly maneuvering targets and lifting-wing quadcopter dynamics, while retaining the computational tractability of an EKF framework.

Specifically, the DC-EKF propagates the state using the following nonlinear dynamics:
\begin{equation}
\left\{
\begin{aligned}
    \hat{\mathbf{q}}_{k+1} &= \delta \mathbf{q} \otimes \mathbf{q}_k, \quad 
    \mathbf{q}_{k+1} = \frac{\hat{\mathbf{q}}_{k+1}}{\|\hat{\mathbf{q}}_{k+1}\|},\\
    \mathbf{p}_{{\rm r},k+1} &= \mathbf{p}_{{\rm r},k} + \mathbf{v}_{{\rm r},k} \Delta t + \tfrac{1}{2} {}^{\rm e}\mathbf{a}_k \Delta t^2,\\
    \mathbf{v}_{{\rm r},k+1} &= \mathbf{v}_{{\rm r},k} + {}^{\rm e}\mathbf{a}_k \Delta t,\\
    {}^{\rm i}\bar{\mathbf{p}}_{k+1} &= {}^{\rm i}\bar{\mathbf{p}}_k + \Delta t \, \mathbf{L}_s({}^{\rm i}\bar{\mathbf{p}}_k, \mathbf{R}_{\rm e}^{\rm c} \mathbf{p}_{{\rm r},k}) 
    \begin{bmatrix} \mathbf{v}_{{\rm r},k}^{\rm T} \mathbf{R}_{\rm c}^{\rm e} & \mathbf{\omega}_k^{\rm T} \mathbf{R}_{\rm c}^{\rm e} \end{bmatrix}^{\rm T},\\
    \mathbf{b}_{{\rm gyr},k+1} &= \mathbf{b}_{{\rm gyr},k} + \mathbf{n}_{\rm gyr} \Delta t,\\
    \mathbf{b}_{{\rm acc},k+1} &= \mathbf{b}_{{\rm acc},k} + \mathbf{n}_{\rm acc} \Delta t,
\end{aligned}
\right.
\label{equ-stateTrans}
\end{equation}
which serve as the prediction model within the EKF framework. Where
\begin{align}
    \delta {\bf q}
      = \left[
          \cos\biggl(\frac{\|{\bf\omega}_k\|\Delta t}{2}\biggr)\;\;
          \frac{{\bf\omega}_k}{\|{\bf \omega}_k\|}
          \sin\biggl(\frac{\|{\bf \omega}_k\|\Delta t}{2}\biggr)
        \right]^{\rm T}.
        \label{equ-delta_q}
\end{align}

\begin{algorithm}[htbp]
\caption{DC-EKF Pseudocode}
\label{alg:DelayedFilter}
\begin{algorithmic}[1]
\Require Image features ${}^{\rm i}\bar{\mathbf{p}}$, IMU rates ${}^{\rm e}\boldsymbol\omega$, accelerations ${}^{\rm e}\mathbf{a}$
\Ensure Estimated states $\hat{\mathbf{x}}_t$
\State Initialize $\mathbf{x}_0$, $\mathbf{P}_0$, and IMU buffer
\For{$t = 1$ \textbf{to} $T^\star$}
    \If{no image features}
        \State Propagate state (Eq.~\eqref{equ-stateTrans}) and covariance; buffer IMU data
    \Else
        \State Delayed update at $t-D$
        \State Re-propagate from $t-D+1$ to $t$ using buffered IMU
    \EndIf
\EndFor
\end{algorithmic}
\end{algorithm}

\section{Lifting-Wing Interception Experiment\label{sec:experiments}}

\subsection{Experimental Setup}
The platform (Fig.~\ref{fig:expParameters}) is a lifting-wing quadcopter equipped with a fixed monocular camera (640$\times$480 px, 120$^{\circ}$ FOV) and onboard real-time processing. Target features were extracted using YOLOv5s\footnote{YOLOv5s: Ultralytics, available at \url{https://github.com/ultralytics/yolov5}} at 20--25 Hz. All perception, estimation, and high-level control algorithms run fully onboard an NVIDIA Jetson Orin NX (16 GB), enabling real-time, self-contained operation without offboard computation.

Experiments involved intercepting a balloon suspended from a quadrotor (Fig.~\ref{fig:expFrames}(a,b)) in outdoor wind conditions of 3--8\,m/s (Beaufort 2--5), where the balloon exhibited rapid, random lateral and vertical maneuvers under the combined influence of wind and the towing quadrotor, \textbf{including motions beyond the quadrotor's own capability}. Two scenarios were tested:

\textbf{Static:} The towing quadrotor hovered while the balloon oscillated in the wind.

\textbf{Dynamic:} The towing quadrotor was manually flown at 6--8\,m/s along arbitrary escape trajectories, adding both motion and aerodynamic disturbances.

Multiple trials were conducted with varying initial separations to ensure statistical reliability (see Table~\ref{tab:interception}).

\subsection{DC-EKF Performance}

Fig.~\ref{fig:expEKF} shows the effectiveness of the DC-EKF filter in the experiment. The figure demonstrates that, when observing a dynamic target, the filter leverages the IMU to accurately predict the target position approximately $0.1\sim0.2\,\rm{s}$ before the visual feature measurements arrive (this value is adjustable depending on the hardware used in this work), thereby effectively compensating for the measurement delay. The filter also addresses frame drops by maintaining continuity of the target position estimate when image features are temporarily lost, ensuring real-time and accurate target position estimation during agile motion.
\vspace{-0.2cm}

\begin{figure}[htbp]
  \centering
  \includegraphics[width=1.0\columnwidth,keepaspectratio]{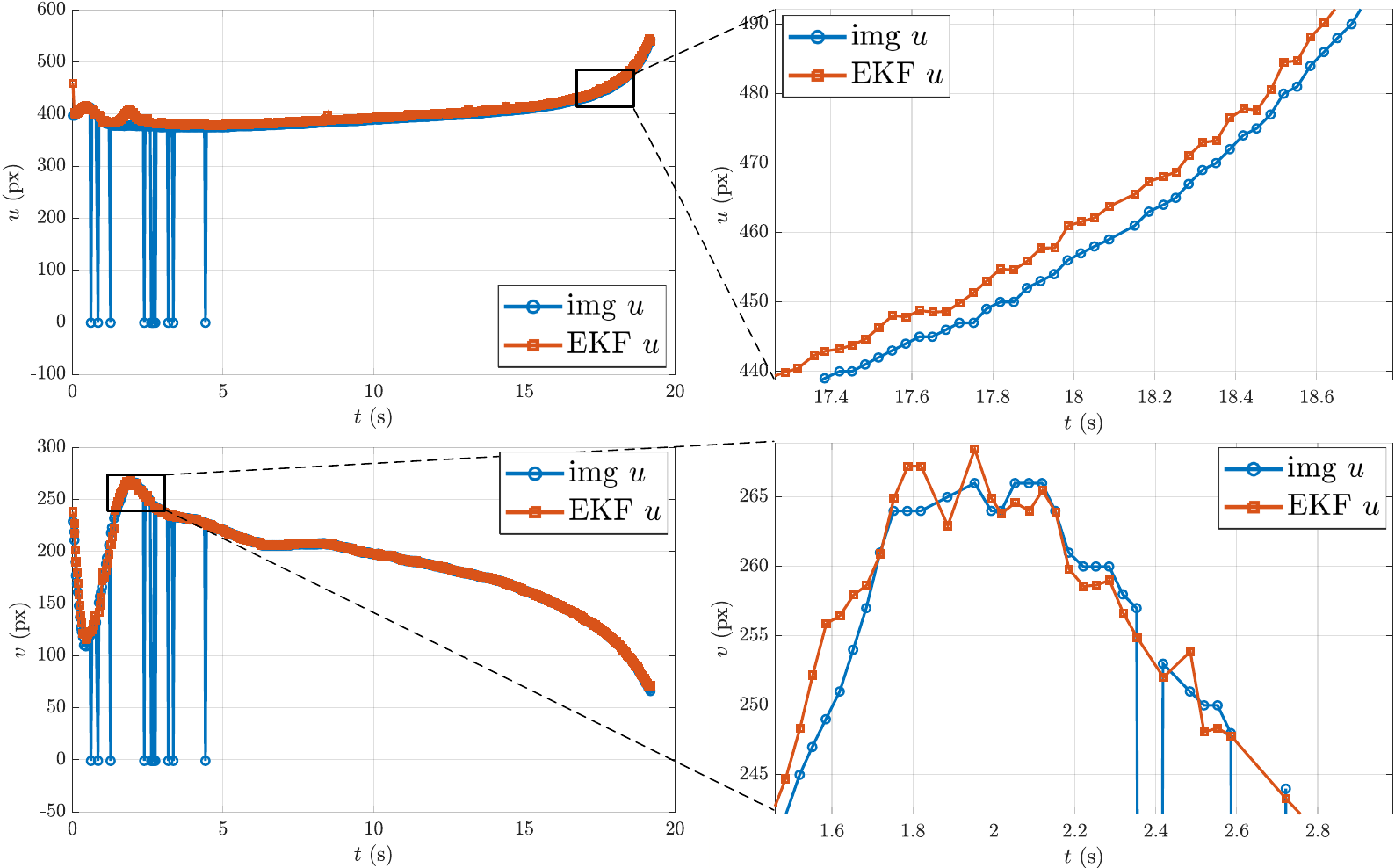}
  \caption{\label{fig:expEKF}DC-EKF performance. The plot shows the temporal evolution of the horizontal and vertical coordinates of the target image features in the IF during an interception, both before and after filtering.}
  \vspace{-0.7cm}
\end{figure}

\subsection{Experimental Data Analysis}

The experimental results of the trials demonstrate a 90\% interception success rate for static targets and a 71\% success rate for dynamic targets (see Table~\ref{tab:interception}).

Notably, in two trials the interceptor successfully achieved precise interceptions of unpredictable dynamic targets exhibiting high-frequency, large-amplitude oscillations at distances of 89\,m and 138\,m (Fig.~\ref{fig:expFrames}(a,b)), thoroughly validating the robustness of the algorithm. Fig.~\ref{fig:expFrames}(e,f) present the velocity-trajectory plots for two interception missions, demonstrating maximum interception velocities of 8.9\,m/s and 10.8\,m/s, respectively. Fig.~\ref{fig:expFrames}(g,h) illustrate the attitude variations during these two interception scenarios, while Fig.~\ref{fig:expFrames}(c,d) depict the distribution of target features within the FOV throughout the interception processes. Collectively, these six subfigures demonstrate that when the target maneuvers, the lifting-wing quadcopter interceptor rapidly executes maneuvering adjustments and achieves swift acceleration to accomplish the subsequent tracking and interception. The controller effectively maintains the unpredictable agile target near the PS-LOS region, ensuring accurate and continuous target tracking.

\subsection{Performance Comparison}
The controller proposed in this work represents an enhancement over the methodology described in \cite{yangHighspeedInterceptionMulticopter2025}. Compared to the previous work, this approach demonstrates significant improvements in interception range and achieves accurate interception capabilities under unpredictable maneuverable target and wind gusting conditions. The detailed comparison is presented in Table~\ref{tab:comparison}.

Compared to \cite{yangHighspeedInterceptionMulticopter2025}, our work yields substantially greater interception range (138 m v.s. 55 m) and improved robustness to target agility and wind gusts (Table~\ref{tab:comparison}).

\begin{table}[htbp]
    \centering
    \caption{Interception Performance}
    \label{tab:interception}
    \small
    \setlength{\extrarowheight}{3pt}
    \setlength{\tabcolsep}{6pt}
    \begin{tabular}{@{} p{1.2cm} p{1.6cm} p{1.0cm} p{1.2cm} p{1.3cm} @{}}
        \toprule
        \multirow{2}{1.2cm}{\parbox{1.2cm}{Towing Quadcopter}}
        & \multirow{2}{1.6cm}{\parbox{1.6cm}{Interception Distance}}
        & \multirow{2}{1.0cm}{\parbox{1.0cm}{Number of Trials}}
        & \multicolumn{2}{c}{Successful Interceptions} \\
        \cmidrule(lr){4-5}
         & & & Previous Work \cite{yangHighspeedInterceptionMulticopter2025} & \textbf{Proposed} \\
        \midrule
        \multirow{2}{*}{Static}
          & 20--30\,m & 5 & 5 (100\%) & \textbf{5 (100\%)} \\
          & 30--50\,m & 5 & 4 (80\%) & \textbf{4 (80\%)} \\
        \midrule
        \multirow{3}{*}{Dynamic}
          & 30--50\,m & 8 & 4 (50\%) & \textbf{7 (87.5\%)} \\
          & 50--70\,m & 8 & 2 (25\%) & \textbf{6 (75\%)} \\
          & $\geq$ 70\,m & 5 & - & \textbf{3 (60\%)} \\
        \bottomrule
    \end{tabular}
    \vspace{-0.2cm}
\end{table}

\begin{table}[htbp]
    \centering
    \caption{Comparison of Dynamic Target Interception Experiments}
    \label{tab:comparison}
    \small
    \setlength{\extrarowheight}{3pt}
    \setlength{\tabcolsep}{5pt}
    \begin{tabular}{@{} >{\raggedright\arraybackslash}m{2.5cm} 
                        >{\raggedright\arraybackslash}m{2.5cm} 
                        >{\raggedright\arraybackslash}m{2.6cm} @{} }
        \toprule
        Experimental \newline Condition & Previous Work \cite{yangHighspeedInterceptionMulticopter2025} & \textbf{Proposed} \\
        \midrule
        Maximum Interception Distance & 55\,m & \textbf{138\,m} \\
        Towing Quadcopter Speed & 5--6\,m/s \newline (nearly constant) & \textbf{6--8\,m/s \newline (variable)} \\
        Target Agility & Minimal & \textbf{unpredictable and \newline significant} \\
        Wind Gusts & Negligible & \textbf{Beaufort 2--5} \\
        \bottomrule
    \end{tabular}
\end{table}

\vspace{-0.3cm}
\section{Conclusion}
\label{sec:conclusion}

This paper introduced a \textit{Planer-Sector Line-of-Sight} guidance method for long-range interception of agile targets using lifting-wing quadcopters. By replacing conventional conic FOV constraints with a planar-sector bound aligned with the camera symmetry plane, PS-LOS significantly enlarges feasible acceleration with reliable target tracking.

For lifting-wing quadcopter interception platform, we designed a \textit{two-layer control architecture}: an outer thrust/LOS law enforcing the sector through barrier terms, and an inner attitude tracker with coordinated-turn compensation. Theoretical analysis using a composite Lyapunov function guarantees \textit{sector invariance and asymptotic convergence} of the relative position. To handle delayed and intermittent monocular sensing, a delay-compensated EKF fuses high-rate IMU propagation with visual updates, providing low-latency target-state estimates.

\textit{Experimental validation} demonstrated persistent interceptions up to 138\,m under irregular target maneuvers and Beaufort 2--5 winds, achieving 90\% success in static and 71\% in dynamic scenarios. Compared with quadrotor-based IBVS baselines, our approach substantially extends range and robustness.

Future work will focus on enabling longer-horizon and more accurate target-state prediction, coupled with a more systematic pursuit planning approach.

\bibliographystyle{IEEEtran}
\bibliography{LwAtt}

\end{document}